\documentclass{article}

\usepackage{microtype}
\usepackage{graphicx}
\usepackage{subfigure}
\usepackage{booktabs}
\usepackage{amsmath}
\usepackage{xcolor}

\usepackage{amsthm}
\usepackage{amssymb}

\newtheorem{theorem}{Theorem}

\theoremstyle{definition}

\usepackage{hyperref}

\usepackage[accepted]{mlsys2025}

\mlsystitlerunning{CurvaDion: Curvature-Adaptive Distributed Orthonormalization}

\begin{document}

\twocolumn[
\mlsystitle{CurvaDion: Curvature-Adaptive Distributed Orthonormalization}




\begin{mlsysauthorlist}
\mlsysauthor{Bhavesh Kumar}{uw,nous}
\mlsysauthor{Roger Jin}{nous}
\mlsysauthor{Jeffrey Quesnelle}{nous}
\end{mlsysauthorlist}

\mlsysaffiliation{uw}{University of Washington, Seattle, Washington, USA}
\mlsysaffiliation{nous}{Nous Research}

\mlsyscorrespondingauthor{Bhavesh Kumar}{bkumar2@uw.edu}
\mlsyscorrespondingauthor{Roger Jin}{roger@nousresearch.com}
\mlsyscorrespondingauthor{Jeffrey Quesnelle}{jfquesne@umich.edu}

\mlsyskeywords{Machine Learning, MLSys}

\vskip 0.3in

\begin{abstract}

As language models scale to trillions of parameters, distributed training across many GPUs becomes essential, yet gradient synchronization over high-bandwidth, low-latency networks remains a critical bottleneck. While recent methods like Dion reduce per-step communication through low-rank updates, they synchronize at every step regardless of the optimization landscape. We observe that synchronization requirements vary dramatically throughout training: workers naturally compute similar gradients in flat regions, making frequent synchronization redundant, while high-curvature regions require coordination to prevent divergence. We introduce CurvaDion, which uses Relative Maximum Momentum Change (RMMC) to detect high-curvature regions requiring synchronization. RMMC leverages momentum dynamics which are already computed during optimization as a computationally tractable proxy for directional curvature, adding only $\mathcal{O}(d)$ operations per layer. We establish theoretical connections between RMMC and loss curvature and demonstrate that CurvaDion achieves 99\% communication reduction while matching baseline convergence across models from 160M to 1.3B parameters. 

\end{abstract}
]



\printAffiliationsAndNotice{}

\section{Introduction}
Training Large Language Models (LLMs) requires distributed computation across multiple workers to handle the scale of modern architectures. As language models grow in size, gradient synchronization increasingly dominates training time: communication costs scale linearly with parameter count while computation can be parallelized across workers. This communication bottleneck fundamentally limits training throughput on bandwidth-constrained networks, determines operational costs in cloud environments, and restricts the feasibility of training across geographically distributed datacenters.

\subsection{Motivation}

Recent advances in optimizer design have demonstrated the effectiveness of orthogonalized momentum updates. Muon~\citep{jordan2024muon}, which applies orthogonalization to momentum at each step, has gained rapid adoption in the LLM training community, achieving state-of-the-art results on models ranging from hundreds of millions to billions of parameters. Muon's success is particularly relevant for distributed training, as its orthogonalization approach naturally extends to communication-efficient methods. Dion \cite{ahn2025dion} achieves communication efficiency through low-rank orthonormalized updates combined with distributed computation. However, it synchronizes the gradients at every training step regardless of the optimization landscape. Algorithm~\ref{alg:dion_baseline} shows the standard Dion procedure, with the synchronization points highlighted. In each iteration, the workers perform two reduction operations in total (lines 5 and 8) to compute the factorization of low rank $M_t \approx P_t R_t^\top$, plus additional communication during orthogonalization. While this ensures convergence equivalent to full-rank methods, it does not exploit the varying importance of updates throughout training.

\begin{algorithm}[h]
\caption{Dion (Baseline) - All steps synchronized}
\label{alg:dion_baseline}
\begin{algorithmic}[1]
\STATE \textbf{Input:} Learning rate $\eta$, momentum $\mu$, rank $r$
\FOR{$t = 1$ to $T$}
    \STATE Compute local gradient $\hat{G}_t^{(i,j)}$
    \STATE $\hat{B}_t^{(i,j)} \gets \hat{M}_{t-1}^{(i,j)} + \hat{G}_t^{(i,j)}$ \hfill // Update buffer
    \STATE $P_t^{(j)}, R_t^{(i)} \gets \text{PowerIter}(\hat{B}_t^{(i,j)}; Q_{t-1}^{(i)})$
    \STATE {$P_t^{(j)} \gets \textsc{AllReduce}_{DP}(P_t^{(j)})$} \hfill // \textbf{Sync 1}
    \STATE $P_t^{(j)} \gets \text{Orthogonalize}(P_t^{(j)})$ \hfill // Contains syncs
    \STATE {$R_t^{(i)} \gets \textsc{AllReduce}_{DP}(R_t^{(i)})$} \hfill // \textbf{Sync 2}
    \STATE $\hat{M}_t^{(i,j)} \gets \hat{B}_t^{(i,j)} - (1-\mu) P_t^{(j)} (R_t^{(i)})^\top$
    \STATE $Q_t^{(i)} \gets \text{ColumnNormalize}(R_t^{(i)})$
    \STATE $X_t^{(i,j)} \gets X_{t-1}^{(i,j)} - \eta \sqrt{m/n} \, P_t^{(j)} (Q_t^{(i)})^\top$
\ENDFOR
\end{algorithmic}
\end{algorithm}

Examining Algorithm~\ref{alg:dion_baseline}, we observe that Dion synchronizes workers at every training step (lines 6, 7, and 8), regardless of the optimization landscape. This uniform synchronization strategy accumulates substantial communication overhead, particularly problematic as model size and worker count scale. However, not all synchronization steps contribute equally to convergence. During stable optimization phases when traversing relatively flat regions of the loss landscape, workers naturally compute similar gradients. In these regions, frequent synchronization provides diminishing returns, as each synchronization incurs high communication costs while yielding marginal coordination benefits. Conversely, during critical periods such as when moving through high-curvature regions of the loss surface, synchronization becomes essential to prevent worker divergence and maintain training stability. The geometry of the loss landscape, particularly its curvature, fundamentally impacts generalization and training dynamics~\citep{foret2021sharpnessaware}, and successful optimization trajectories must either avoid or navigate out of high-curvature regions to prevent training instability~\citep{gilmer2021loss}. In distributed training, the synchronization frequency should adapt to these optimization dynamics: too-frequent synchronization in flat regions wastes communication bandwidth, while insufficient synchronization in curved regions causes worker divergence that degrades convergence quality~\citep{lin2020dont, ortiz2023quadratic}. These findings motivate an adaptive synchronization approach that concentrates the communication budget where it matters most, in regions of high curvature.

\subsection{Contributions}

We make the following contributions:
\begin{enumerate}
    \item We propose RMMC (Relative Maximum Momentum Change), a computationally efficient metric that uses momentum dynamics, already maintained during optimization, to detect high-curvature regions requiring synchronization, adding only $\mathcal{O}(d)$ operations per layer.
    \item We introduce CurvaDion, an adaptive synchronization framework that extends Dion with momentum-based triggering, and establish theoretical connections between RMMC and directional curvature (Theorem~\ref{thm:rmmc-curvature}).
    \item We demonstrate that CurvaDion achieves 99\% communication reduction while matching baseline convergence across models from 160M to 1.3B parameters, with projected speedups of low latency networks.
\end{enumerate}

\section{Methodology}
\label{sec:method}
The central challenge is identifying when synchronization is necessary without computing expensive second-order information. The most direct approach would be to measure loss curvature via the Hessian matrix $\nabla^2 F(\theta)$, whose eigenvalues quantify curvature along different directions in parameter space. However, for neural networks with $d$ parameters, computing the full Hessian requires $\mathcal{O}(d^2)$ memory and $\mathcal{O}(d^3)$ operations for inversion, prohibitively expensive for modern large-scale models~\citep{martens2015optimizing, dangel2020modular}. Even approximate second-order methods like K-FAC~\citep{martens2015optimizing, grosse2016kroneckerfactored} or quasi-Newton approaches~\citep{nocedal2006numerical} introduce substantial computational overhead, making them impractical for frequent curvature monitoring at every training step. While Hessian-vector products can be computed more efficiently in $\mathcal{O}(d)$ time using automatic differentiation~\citep{pearlmutter1994fast}, this still requires an additional backward pass per evaluation, effectively doubling the computational cost per step. For distributed training where communication already limits throughput, adding 2$\times$ computational overhead is untenable. Moreover, Hessian approximations face additional challenges: they can be numerically unstable, require careful hyperparameter tuning, and may not capture the relevant directional curvature along the optimization trajectory~\citep{yao2021adahessian, zhang1998blockdiagonal}.

\subsection{RMMC: Relative Maximum Momentum Change}
\label{sec:rmmc}

Our key insight is that we can leverage momentum dynamics, already computed during Dion, to encode curvature information. The momentum vector $M_t = \mu M_{t-1} + G_t$, maintained for optimization, accumulates an exponentially weighted history of gradients with decay rate $(1-\mu)$. This exponential averaging provides natural variance reduction: the variance of momentum is approximately $(1-\mu^2)^{-1}$ times smaller than that of the raw gradient~\citep{sutskever2013importance}. Crucially, recent work shows that momentum magnitude changes correlate strongly with loss curvature, as the curvature in the momentum direction captures the effective local geometry experienced by the optimizer, and momentum-based perturbations naturally measure the sharpness along the optimization trajectory~\citep{kunstner2024momentum}. When the loss surface exhibits high curvature, gradient magnitudes change rapidly along the optimization path, causing corresponding rapid changes in momentum magnitude~\citep{gilmer2021loss}. Conversely, in flat regions with low curvature, gradients remain stable and momentum evolves slowly. This relationship suggests that monitoring momentum magnitude changes can serve as a computationally cheap proxy for detecting regime changes that require coordination.

Formally, we define this concept as Relative Maximum Momentum Change (RMMC) for layer $\ell$ at step $t$:
\begin{equation}
\label{eq:rmmc}
\text{RMMC}_\ell(t) = \frac{\big| \|M_\ell(t)\| - \|M_\ell(t-1)\| \big|}{\|M_\ell(t-1)\|}
\end{equation}

RMMC provides several advantages for distributed training. Computing RMMC requires only $\mathcal{O}(d)$ operations per layer (two norm computations), negligible compared to the $\mathcal{O}(d^2)$ cost of gradient computation, and adds less than 1\% per-step overhead in practice. The metric naturally aggregates across layers using a single all-reduce operation that determines the global maximum RMMC, requiring communication of just one scalar per worker ($<$1 KB total). We aggregate via maximum rather than mean because worker divergence can be driven by a single layer entering a high-curvature region, with this divergence then amplifying through subsequent layers during forward propagation~\citep{tang2025dreamddp}. The absolute value in Equation~\ref{eq:rmmc} captures both increases and decreases in momentum magnitude: increases signal entry into high-curvature regions requiring immediate coordination, while decreases indicate convergence phases or regime transitions where consolidating worker progress prevents drift. This is critical because both scenarios benefit from synchronization, eliminating the need to distinguish between curvature-driven and dynamics-driven events.

Section~\ref{sec:theory} formalizes the connection between RMMC and directional curvature via Theorem~\ref{thm:rmmc-curvature}. Unlike fixed-interval synchronization schedules employed by Local SGD variants~\citep{stich2018local, douillard2023diloco}, which can cause divergence when workers traverse high-curvature regions between scheduled syncs, RMMC triggers coordination precisely when the optimization landscape demands it. The threshold parameter $\tau$ provides a natural trade-off: larger values reduce communication frequency but risk worker divergence, while smaller values increase synchronization overhead. In our experiments (Section~\ref{sec:experiments}), we find $\tau \in [0.1, 0.8]$ provides robust performance across different model scales and optimization phases, achieving up to 98\% reduction in synchronization frequency while matching the convergence quality of full synchronous training.

\subsection{CurvaDion}
\label{sec:curvaDion}

We now present CurvaDion (Algorithm~\ref{alg:curvaDion}), which extends Dion with adaptive synchronization based on RMMC. At each training step $t$, workers independently compute per-layer RMMC values (lines 6--9) using only local information. A single all-reduce operation aggregates the global maximum RMMC across all layers and workers (line 10), adding negligible overhead (one scalar per worker, $<$1 KB total). When this global maximum exceeds threshold $\tau$, workers execute the full Dion synchronization protocol (lines 13--19): computing low-rank factorizations via power iteration, performing orthogonalization with distributed all-reduce operations, and updating parameters. Otherwise, workers continue training locally using standard gradient updates (lines 22--23), avoiding all communication.

\begin{algorithm}[t]
\caption{CurvaDion - Selective synchronization}
\label{alg:curvaDion}
\begin{algorithmic}[1]
\STATE \textbf{Input:} Learning rate $\eta$, momentum $\mu$, rank $r$, threshold $\tau$
\STATE Initialize: $M_0^{(i,j)} \gets 0$, $\text{prev\_norm}_\ell \gets 0$ for all layers $\ell$
\FOR{$t = 1$ to $T$}
    \STATE Compute local gradient $\hat{G}_t^{(i,j)}$
    \STATE $\hat{B}_t^{(i,j)} \gets \hat{M}_{t-1}^{(i,j)} + \hat{G}_t^{(i,j)}$ \hfill // Local update
    \FOR{each layer $\ell$}
        \STATE $\text{RMMC}_\ell \gets \frac{|\|\hat{B}_t^\ell\| - \text{prev\_norm}_\ell|}{\text{prev\_norm}_\ell + \epsilon}$
        \STATE $\text{prev\_norm}_\ell \gets \|\hat{B}_t^\ell\|$
    \ENDFOR
    \STATE $\text{global\_max\_RMMC} \gets \textsc{AllReduce}_{\max}(\max_\ell \text{RMMC}_\ell)$
    \IF{$\text{global\_max\_RMMC} > \tau$}
        \STATE // \textbf{Sync triggered - perform full Dion update}
        \STATE $P_t^{(j)}, R_t^{(i)} \gets \text{PowerIter}(\hat{B}_t^{(i,j)}; Q_{t-1}^{(i)})$
        \STATE $P_t^{(j)} \gets \textsc{AllReduce}_{DP}(P_t^{(j)})$
        \STATE $P_t^{(j)} \gets \text{Orthogonalize}(P_t^{(j)})$
        \STATE $R_t^{(i)} \gets \textsc{AllReduce}_{DP}(R_t^{(i)})$
        \STATE $\hat{M}_t^{(i,j)} \gets \hat{B}_t^{(i,j)} - (1-\mu) P_t^{(j)} (R_t^{(i)})^\top$
        \STATE $Q_t^{(i)} \gets \text{ColumnNormalize}(R_t^{(i)})$
        \STATE $X_t^{(i,j)} \gets X_{t-1}^{(i,j)} - \eta \sqrt{m/n} \, P_t^{(j)} (Q_t^{(i)})^\top$
    \ELSE
        \STATE // \textbf{Local update only - no synchronization}
        \STATE $\hat{M}_t^{(i,j)} \gets \hat{B}_t^{(i,j)}$
        \STATE $X_t^{(i,j)} \gets X_{t-1}^{(i,j)} - \eta_{\text{local}} \hat{G}_t^{(i,j)}$
    \ENDIF
\ENDFOR
\end{algorithmic}
\end{algorithm}

CurvaDion makes synchronization contingent on the optimization landscape rather than predetermined schedules. This design provides three key advantages for distributed systems. First, the adaptive triggering mechanism concentrates expensive communication operations during critical training phases (high curvature, regime transitions) while eliminating redundant synchronization in stable regions, achieving substantial bandwidth savings without compromising convergence. Second, the decision logic adds minimal computational overhead: per-layer norm computations require only $\mathcal{O}(d)$ operations, and the max-reduce operation for the sync decision completes in logarithmic rounds with respect to worker count. Third, CurvaDion integrates seamlessly with existing data-parallel training frameworks, requiring only the addition of RMMC computation and a conditional branch around existing synchronization primitives. The threshold $\tau$ provides a single, interpretable hyperparameter that directly controls the communication-convergence trade-off, with our experiments demonstrating robust performance across $\tau \in [0.1, 0.8]$ for diverse model architectures and scales. Unlike fixed-interval methods that must carefully tune synchronization frequency $H$ for each specific training configuration~\citep{lin2020dont, ortiz2023quadratic}, CurvaDion's adaptive mechanism automatically adjusts to the local optimization landscape, making it more broadly applicable across different hardware configurations and network conditions.

\section{Theoretical Results}
\label{sec:theory}

We now establish the theoretical basis for using RMMC as a synchronization trigger. In particular, we show that RMMC approximates the directional 
curvature of the loss function along the momentum direction, scaled by the 
learning rate, plus a bias term that captures transient optimization dynamics. 
This establishes RMMC as a principled, computationally efficient proxy for 
detecting high-curvature regions that require worker synchronization.

\begin{theorem}[RMMC-Curvature Relationship]
\label{thm:rmmc-curvature}
Under $L$-Lipschitz continuous gradients, $\|\nabla F(x) - \nabla F(y)\| \leq L\|x-y\|$ for all $x, y$, and small step size satisfying $\eta\|M_{t-1}\| = O(1/L)$, the Relative Maximum Momentum Change satisfies:
\begin{equation}
\text{RMMC}_t = \left|\eta \kappa_M(x_{t-1}) + B_t\right| + O(\eta^2)
\end{equation}
where $\kappa_M(x) = v^T \nabla^2 F(x) v$ is the directional curvature in the momentum direction $v = M/\|M\|$, and $B_t = (1-\mu) - \frac{v_{t-1}^T \nabla F(x_{t-1})}{\|M_{t-1}\|}$ quantifies whether momentum has stabilized relative to the gradient scale.
\end{theorem}

\begin{proof}[Proof Sketch]
We decompose momentum as $M_t = m_t v_t$ where $m_t = \|M_t\|$ and $v_t = M_t/\|M_t\|$. From the momentum update $M_t = \mu M_{t-1} + \nabla F(x_t)$, we compute $\|M_t\|^2 = \mu^2 m_{t-1}^2 + 2\mu M_{t-1}^T \nabla F(x_t) + \|\nabla F(x_t)\|^2$. Taking the square root and applying Taylor approximation for small perturbations yields:
\begin{equation*}
m_t - m_{t-1} = (\mu - 1)m_{t-1} + v_{t-1}^T \nabla F(x_t) + O(\eta^2 m_{t-1})
\end{equation*}

Since $x_t = x_{t-1} - \eta M_{t-1}$, we Taylor expand the gradient:
\begin{equation*}
\nabla F(x_t) = \nabla F(x_{t-1}) - \eta \nabla^2 F(x_{t-1}) M_{t-1} + O(\eta^2 m_{t-1}^2)
\end{equation*}

Projecting onto the momentum direction and using $M_{t-1} = m_{t-1} v_{t-1}$:

\begin{align*}
v_{t-1}^T \nabla F(x_t) &= v_{t-1}^T \nabla F(x_{t-1}) - \eta m_{t-1} \kappa_M(x_{t-1}) \\
&\quad + O(\eta^2 m_{t-1})
\end{align*}

where $\kappa_M(x) = v^T \nabla^2 F(x) v$ is the directional curvature. Substituting this into the momentum change and dividing by $m_{t-1}$:
\begin{align*}
\frac{m_t - m_{t-1}}{m_{t-1}} &= (\mu-1) + \frac{v_{t-1}^T \nabla F(x_{t-1})}{m_{t-1}} \\
&\quad - \eta\kappa_M(x_{t-1}) + O(\eta^2)
\end{align*}

Taking the absolute value and defining $B_t = (1-\mu) - \frac{v_{t-1}^T \nabla F(x_{t-1})}{m_{t-1}}$ gives the result. The complete proof appears in Appendix~\ref{app:proof-theorem1}.
\end{proof}

A key design choice in Theorem~\ref{thm:rmmc-curvature} is measuring curvature $\kappa_M(x) = v^T \nabla^2 F(x) v$ along the momentum direction $v = M/\|M\|$ rather than the instantaneous gradient direction $\nabla F/\|\nabla F\|$. The momentum direction is substantially less noisy due to exponential averaging, with variance reduced by approximately $(1-\mu^2)^{-1}$ compared to instantaneous gradients~\citep{sutskever2013importance}. Beyond noise reduction, the momentum vector $M$ encodes the optimizer's effective trajectory by incorporating historical information about the loss landscape rather than instantaneous local information from a single mini-batch~\citep{kunstner2024momentum}. For distributed training where workers operate on different data subsets and must detect optimization events requiring coordination based on local information, curvature measured along this stable, trajectory-aligned momentum direction provides a more reliable synchronization signal than gradient-based curvature measures~\citep{kunstner2024momentum}.

Theorem~\ref{thm:rmmc-curvature} also reveals that RMMC is a composite signal encoding both geometric complexity through directional curvature $\kappa_M$ and optimization dynamics through the bias term $B_t$. The bias term $B_t = (1-\mu) - \frac{v_{t-1}^T \nabla F(x_{t-1})}{m_{t-1}}$ quantifies deviation from momentum-gradient equilibrium. In steady state, the momentum update $M = \mu M + \nabla F$ implies $M = \frac{\nabla F}{1-\mu}$, so that $m = \frac{\|\nabla F\|}{1-\mu}$. With momentum-gradient alignment where $v^T \nabla F \approx \|\nabla F\|$, we obtain:
\begin{equation*}
B_t \approx (1-\mu) - \frac{\|\nabla F\|}{m} \approx (1-\mu) - \frac{\|\nabla F\|}{\|\nabla F\|/(1-\mu)} = 0
\end{equation*}

When $B_t \approx 0$, momentum has stabilized and $\text{RMMC}_t \approx \eta|\kappa_M(x_{t-1})|$ directly measures curvature. Conversely, when $|B_t|$ is large, the system is far from equilibrium, indicating transient dynamics during momentum buildup or gradient direction changes. Thus RMMC is low when the landscape is flat and momentum is equilibrated, but high when either significant curvature or dynamic instability is present. This dual sensitivity eliminates the need to distinguish between curvature-driven and dynamics-driven synchronization events, as both benefit from worker coordination. 

\section{Numerical Results}
\label{sec:experiments}

We evaluate CurvaDion on language model pretraining using a 160M-parameter GPT-style transformer \citep{radford2019language} with 12 layers, 768-dimensional embeddings, and 6 attention heads. The model uses rotary position embeddings \citep{su2021roformer}, RMS normalization \citep{zhang2019root}, and squared ReLU activations \citep{so2021primer}. We train on FineWeb \citep{penedo2024fineweb} with sequence length 1024, global batch size 1024 distributed across workers with per-device batch size 32 using PyTorch's Distributed Data Parallel (DDP) \citep{li2020pytorch}. This training was done on a single 8xH100 node. Training uses learning rate $\eta = 0.02$, momentum $\mu = 0.95$, weight decay 0.01, and Dion rank fraction $r = 0.125$. All experiments run for 3000 iterations with 20\% warmdown, using mixed precision (bfloat16) \citep{micikevicius2017mixed}.

\subsection{Convergence Quality and Communication Overhead}
\label{sec:main_results}

We demonstrate that CurvaDion achieves convergence quality equivalent to baseline Dion while dramatically reducing communication overhead. Figure~\ref{fig:main_results} (left) shows final validation perplexity across threshold values $\tau \in [0.1, 0.9]$. CurvaDion maintains perplexity $\approx 28.98$ across all thresholds, matching Dion's baseline performance. This demonstrates that adaptive synchronization based on RMMC preserves convergence quality despite synchronizing substantially less frequently than the baseline, which coordinates workers at every training step.

The communication savings are substantial. Figure~\ref{fig:main_results} (right) quantifies per-step communication volume for Dion and CurvaDion with various thresholds. Baseline Dion requires 619 MB per training step to synchronize workers. CurvaDion reduces this to 18.6 MB/step for $\tau = 0.1$ (97\% reduction), 5.2 MB/step for $\tau = 0.3$ (99\% reduction), and 4.3 MB/step for $\tau = 0.7$ (99\% reduction). Over the full 3000-step training run, this translates to total communication of 1813 GB for Dion versus 13-54 GB for CurvaDion depending on threshold choice, representing up to 143$\times$ reduction in total communication volume. These results confirm that CurvaDion's adaptive synchronization mechanism successfully identifies when worker coordination is necessary, concentrating communication during critical optimization phases while eliminating redundant synchronization in stable regions.

\begin{figure*}[t]
    \centering
    \includegraphics[width=0.48\textwidth]{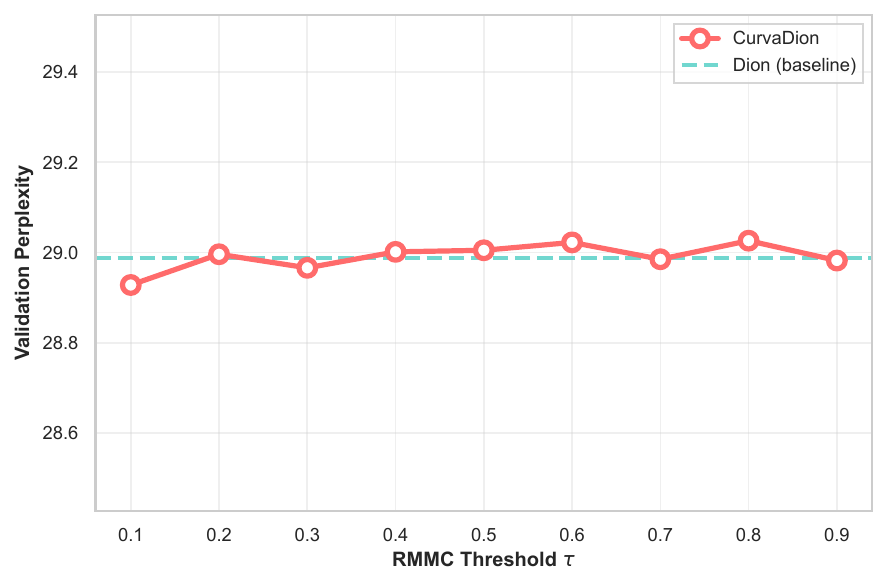}
    \hfill
    \includegraphics[width=0.48\textwidth]{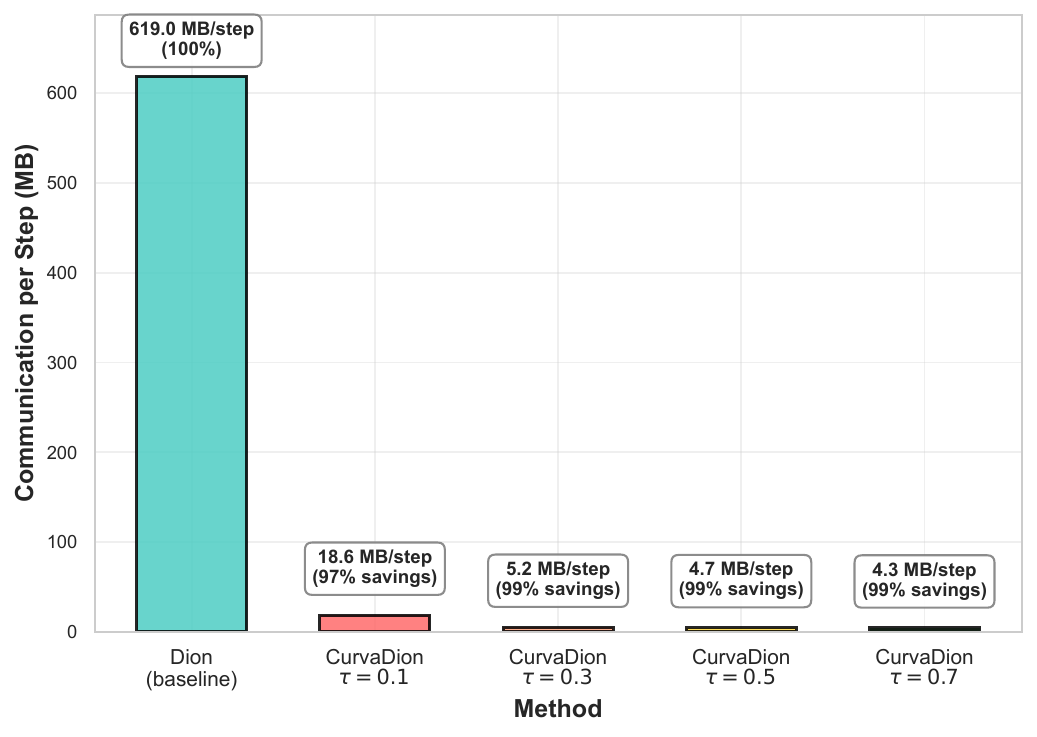}
    \caption{\textit{Left:} Final validation perplexity across RMMC thresholds $\tau$. CurvaDion matches baseline Dion performance ($\approx 28.98$) across all thresholds. \textit{Right:} Per-step communication volume. CurvaDion achieves 97-99\% reduction compared to baseline Dion's 619 MB/step.}
    \label{fig:main_results}
\end{figure*}

\subsection{Performance at Scale}
\label{sec:scaling}

To evaluate CurvaDion's effectiveness at production scale, we conduct experiments across three model sizes spanning 160M to 1.3B parameters, with training data ranging from 3.1B to 100B tokens from FineWeb. We fix the threshold $\tau = 0.5$ across all configurations to demonstrate robustness without per-scale tuning. All experiments maintain consistent hyperparameters: learning rate $\eta=0.02$ with 1\% warmup and 10\% warmdown, momentum $\mu=0.95$, weight decay $0.01$, and Dion rank fraction $r=0.125$. The experimental configurations are detailed in Table~\ref{tab:scale_configs}.

\begin{table}[t]
\centering
\label{tab:scale_configs}
\begin{tabular}{lccc}
\toprule
\textbf{Model} & \textbf{$d_{\text{model}}$ / Layers / Heads} & \textbf{Steps} & \textbf{Tokens} \\
\midrule
160M  & 768 / 12 / 6   & 3,000  & 3.1B \\
350M  & 1024 / 24 / 32 & 9,000  & $\sim 20B$ \\
1.3B  & 2048 / 24 / 32 & 47,000 & $ \sim100B$ \\
\bottomrule
\end{tabular}
\caption{Scaling study configurations. All models use global batch size 2.1M tokens (160M uses 1.0M), learning rate $\eta=0.02$, momentum $\mu=0.95$, and Dion rank fraction $r=0.125$.}
\end{table}

\begin{table*}[t]
\centering
\label{tab:scale_results}
\begin{tabular}{lccccccc}
\toprule
\textbf{Model} & \begin{tabular}{@{}c@{}}\textbf{CurvaDion}\\\textbf{Val PPL}\end{tabular} & \begin{tabular}{@{}c@{}}\textbf{Dion}\\\textbf{Val PPL}\end{tabular} & \begin{tabular}{@{}c@{}}\textbf{Sync}\\\textbf{Rate}\end{tabular} & \begin{tabular}{@{}c@{}}\textbf{Comms}\\\textbf{Dion}\end{tabular} & \begin{tabular}{@{}c@{}}\textbf{Comms}\\\textbf{CurvaDion}\end{tabular} & \begin{tabular}{@{}c@{}}\textbf{Comms}\\\textbf{Saved}\end{tabular} \\
\midrule
160M  & 28.97 & 28.98 & 1.8\% & 1,813 GB  & 14 GB  & 99.2\% \\
350M  & 19.75 & 19.73 & 1.2\% & 13,905 GB  & 94 GB  & 99.3\% \\
1.3B  & 14.61 & 13.84 & 1.0\% & 216,044 GB & 2,160 GB & 99.0\% \\
\bottomrule
\end{tabular}
\caption{Scaling results with $\tau = 0.5$. CurvaDion matches baseline Dion convergence while dramatically reducing communication overhead. Communication volumes are cumulative over the full training run.}
\end{table*}

Table~\ref{tab:scale_results} shows that CurvaDion achieves $\approx 99\%$communication reduction while matching baseline Dion convergence quality. The 1.3B model trained on 100B tokens is particularly significant, as this configuration is comparable to publicly released models like OPT-1.3B trained on 180B tokens~\citep{zhang2022opt}.

Full training and validation curves for all three scales appear in Appendix~\ref{app:scaling_runs}.

\subsection{Synchronization Dynamics}
\label{sec:sync_dynamics}

We demonstrate that when workers synchronize matters more than how often by comparing CurvaDion against Scheduled Sync Dion, a DiLoCo-inspired baseline~\cite{douillard2023diloco} that runs local Dion updates for $H$ steps before synchronizing. Specifically, Scheduled Sync Dion follows the same algorithm as CurvaDion (Algorithm~2) but replaces the RMMC-based triggering condition (line 11) with a fixed schedule: workers perform local updates (lines 22-23) for $H-1$ steps, then execute the full Dion synchronization protocol (lines 13-19) on every $H$-th step. We compare against Scheduled Sync Dion rather than other adaptive synchronization methods~\cite{Tyagi_2023,douillard2023diloco} because those methods are fundamentally incompatible with Dion's orthonormalized momentum structure. Preliminary experiments showed that applying adaptive synchronization rules from both SelSync~\cite{Tyagi_2023} and DiLoCo~\cite{douillard2023diloco} to Dion results in training divergence (final validation perplexity $>100$ vs. 28.97 for CurvaDion). This divergence occurs because existing adaptive methods do not account for the correlation structure in orthonormalized updates, where desynchronization in the low-rank subspace causes instabilities during re-orthogonalization. CurvaDion's momentum-based triggering naturally adapts to these dynamics. By using Scheduled Sync Dion as our baseline, we isolate synchronization timing as the primary variable while maintaining Dion's synchronization mechanism. We focus on $H=100$ (syncing every 100 steps, approximately 1\% sync rate) compared to CurvaDion with $\tau = 0.3$ (also achieving $\sim$1\% sync rate), ensuring matched communication budgets.

\begin{figure}[t]
    \centering
    \includegraphics[width=\columnwidth]{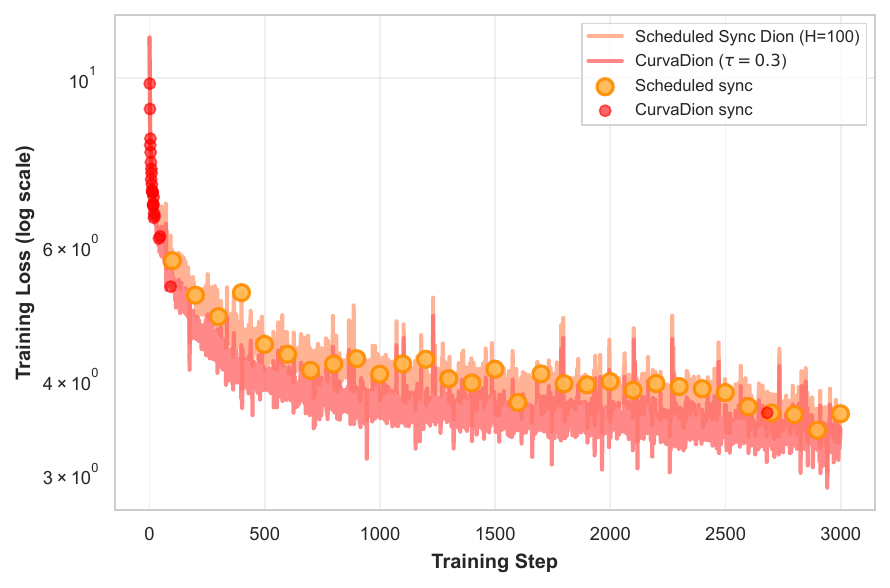}
    \vspace{-2mm}
    \includegraphics[width=\columnwidth]{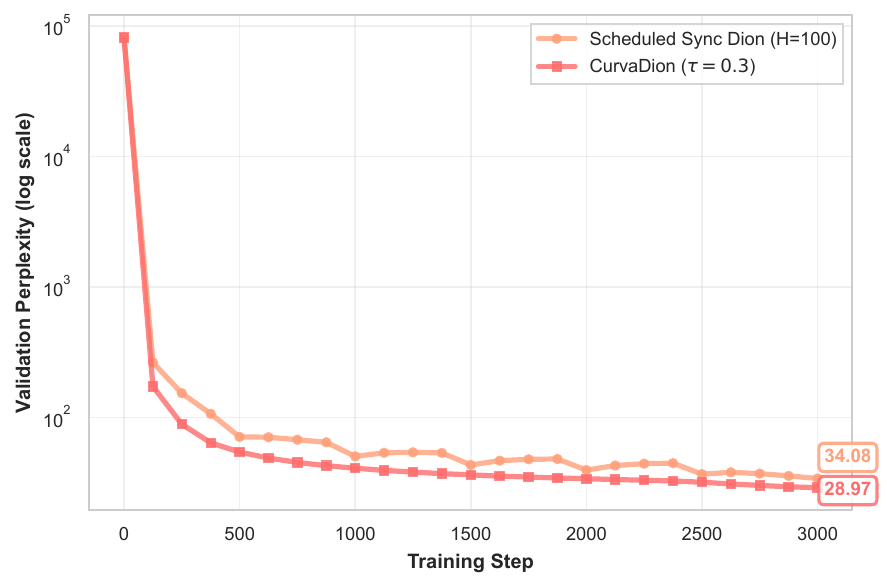}
    \caption{\textbf{Adaptive synchronization outperforms fixed schedules at matched communication budgets.} Scheduled Sync Dion ($H=100$) syncs uniformly every 100 steps while CurvaDion ($\tau=0.3$) adapts to optimization dynamics. Both methods sync $\sim$1\% of steps, but CurvaDion achieves 28.97 final perplexity versus 34.08 for Scheduled Sync Dion.}
    \label{fig:scheduled_comparison}
\end{figure}

Figure~\ref{fig:scheduled_comparison} reveals that despite nearly identical synchronization frequencies, CurvaDion achieves substantially better convergence with final validation perplexity of 28.97 versus 34.08 for Scheduled Sync Dion. This 5.1 point gap demonstrates that the placement of synchronization events, not merely their frequency, determines convergence quality. This pattern generalizes across a wide range of synchronization budgets: Appendix~\ref{app:scheduled_sync} shows that as scheduled synchronization frequency decreases from $H=10$ to $H=250$, performance degrades monotonically from 32.09 to 35.18 perplexity, while CurvaDion maintains consistent $\sim$29 perplexity across all threshold values by adaptively concentrating communication during optimization-critical phases. The training loss curves show similar overall trajectories, but the temporal distribution of synchronization events differs dramatically. Scheduled Sync Dion maintains uniform 100-step intervals regardless of optimization dynamics, visible as evenly spaced markers throughout training. In contrast, CurvaDion concentrates synchronizations heavily during the first 500 steps with dense clusters, then syncs sparsely during the stable convergence phase from steps 500-3000. This adaptive placement ensures worker coordination during critical optimization phases while eliminating redundant communication in flat regions where workers naturally compute similar updates.

The mechanism underlying this adaptive behavior becomes clear when examining RMMC dynamics in Figure~\ref{fig:rmmc_dynamics}. Theorem~1 establishes that $\text{RMMC}_t = |\eta\kappa_M(x_{t-1}) + B_t|$ where $\kappa_M$ measures directional curvature and $B_t = (1-\mu) - \frac{v_{t-1}^T \nabla F(x_{t-1})}{\|M_{t-1}\|}$ quantifies momentum-gradient disequilibrium. During early training (steps 0-200), average RMMC values decay rapidly from $>10$ to $<1$ across all thresholds. This initial spike reflects large $|B_t|$ as momentum builds from initialization and the optimizer navigates transient dynamics before reaching a stable trajectory. The bias term $B_t$ dominates during this phase because momentum has not yet equilibrated with the gradient scale, causing frequent threshold crossings that trigger the dense synchronization pattern observed in Figure~\ref{fig:scheduled_comparison}.

\begin{figure}[t]
    \centering
    \includegraphics[width=\columnwidth]{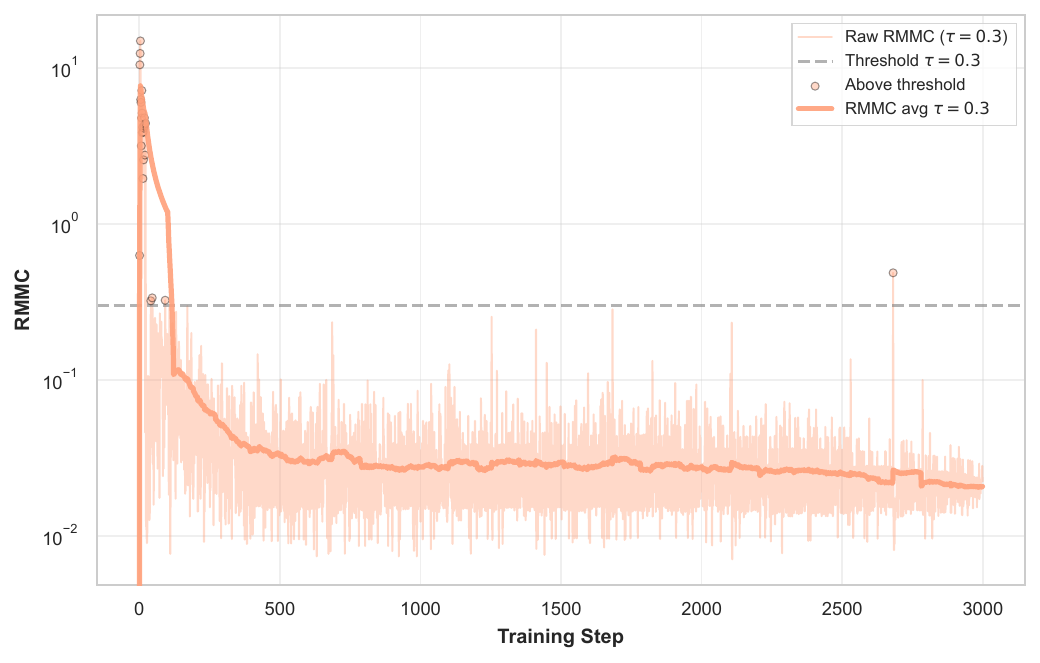}
    \caption{Raw RMMC values show high variance with occasional spikes triggering synchronization ($\tau=0.3$). Average RMMC reveals two-phase behavior: rapid decay during early transient dynamics (steps 0-200, dominated by bias term $B_t$), then stable low values with curvature-driven spikes (steps 200-3000, dominated by $\kappa_M$).}
    \label{fig:rmmc_dynamics}
\end{figure}

As training progresses beyond step 200, average RMMC stabilizes near $0.1$ as momentum equilibrates ($B_t \to 0$) and the optimization enters flatter regions. However, the raw RMMC signal reveals occasional sharp spikes above the threshold $\tau=0.3$ even during this stable phase. These spikes, now driven primarily by the curvature term $\eta\kappa_M$ rather than the bias term $B_t$, trigger synchronizations precisely when the loss landscape exhibits high curvature along the momentum direction. This dual-phase behavior validates our theoretical framework: RMMC naturally concentrates synchronization during both early transient dynamics (high $|B_t|$) and later curvature-driven regime changes (high $|\kappa_M|$), eliminating the need to distinguish between these scenarios. Fixed-interval methods like Scheduled Sync Dion cannot adapt to these optimization dynamics, leading to wasted communication in stable regions and potential worker divergence during high-curvature phases between scheduled syncs.

\subsection{Wall Clock Analysis}
\label{sec:wall_clock_analysis}

To validate the practical impact of our communication savings, we conducted wall clock timing measurements on a two-node distributed training setup. We trained the GPT-160M model for 3000 steps over InfiniBand/RoCE interconnect (measured latency $\sim$0.1ms) with detailed per-step timing instrumentation. Each training step comprises three phases: computation (16 gradient accumulation micro-batches of forward and backward passes), communication (gradient synchronization), and overhead (learning rate scheduling and gradient zeroing).

\begin{figure}[h]
    \centering
    \includegraphics[width=\columnwidth]{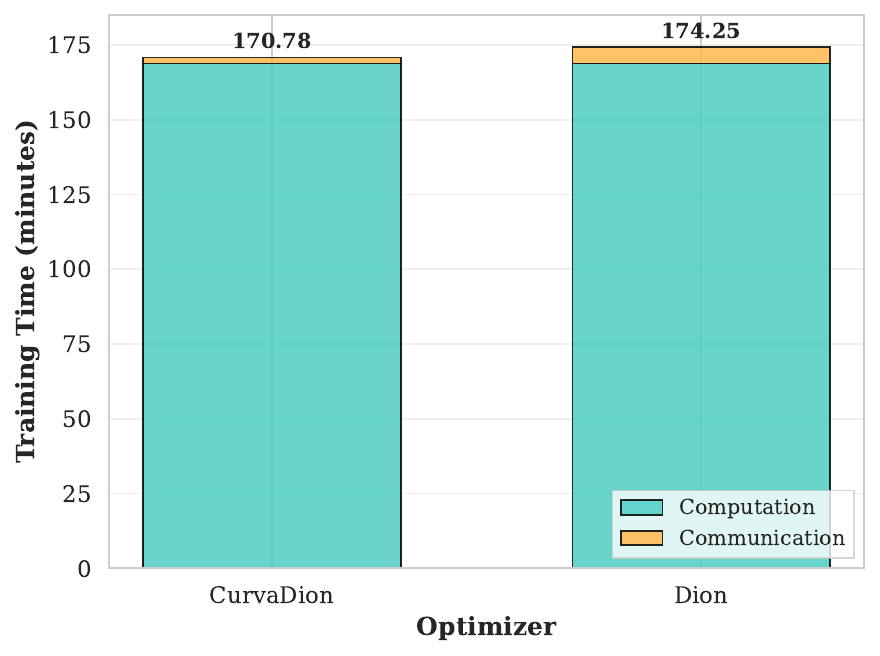}
    \caption{Total training time over 3000 steps. CurvaDion completes training in 170.78 minutes compared to 174.25 minutes for baseline Dion, achieving a 1.02$\times$ speedup (3.47 minutes saved). The modest speedup reflects fast InfiniBand interconnect where gradient synchronization (70ms) represents only 2\% of total step time (3485ms), with computation (3375ms, 96.8\%) dominating.}
    \label{fig:wall_clock_total}
\end{figure}
    
\begin{figure}[h]
    \centering
    \includegraphics[width=\columnwidth]{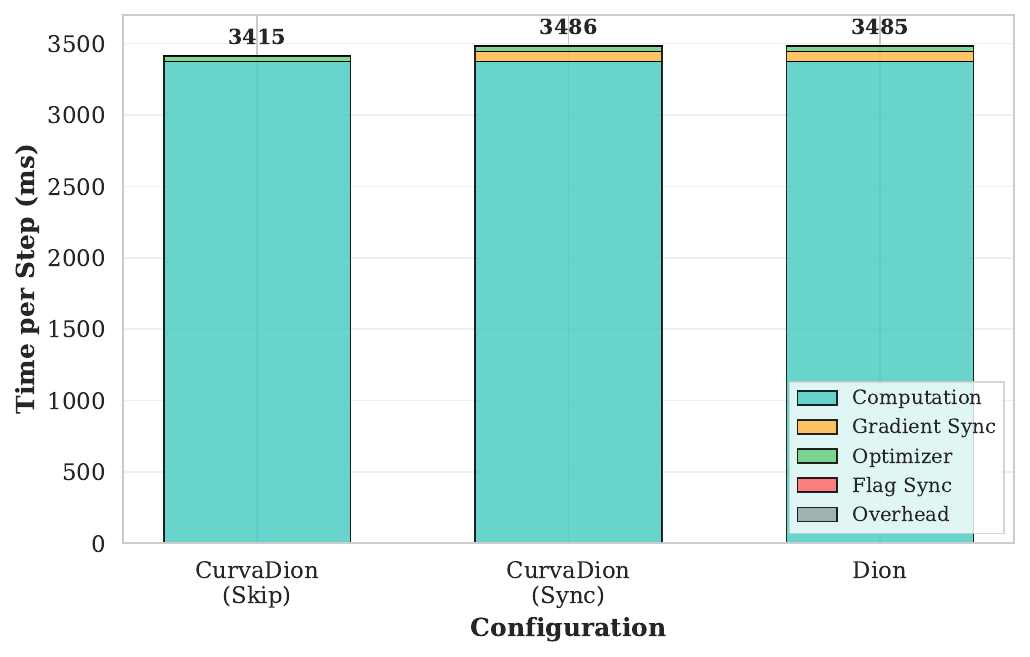}
    \caption{Per-step time breakdown across configurations. Skip steps require 3415ms with only 0.5ms flag overhead (0.015\% of step time), while sync steps require 3486ms including 70ms gradient all-reduce. At 1\% synchronization rate, CurvaDion achieves 99.3\% communication time reduction (70ms $\rightarrow$ 0.5ms) with negligible overhead.}
    \label{fig:wall_clock_breakdown}
\end{figure}

Figure~\ref{fig:wall_clock_total} shows the total training time over 3000 steps. CurvaDion completes training in 170.78 minutes compared to 174.25 minutes for baseline Dion, achieving a 1.02$\times$ speedup and saving 3.47 minutes. While this improvement appears modest, it directly reflects the network characteristics: on our fast InfiniBand connection, gradient synchronization requires only 70ms per step, representing merely 2\% of the total 3485ms step time. The dominant cost is computation (3375ms, or 96.8\%), which remains constant across both methods.

Figure~\ref{fig:wall_clock_breakdown} provides a detailed per-step breakdown. When CurvaDion skips synchronization, step time is 3415ms which consists of 3375ms computation, 39ms optimizer updates, and critically, only 0.5ms for the 4-byte flag all-reduce. This demonstrates that our skip decision mechanism introduces negligible overhead (0.015\% of step time). When synchronization is required, CurvaDion's step time increases to 3486ms, with the additional 70ms for gradient all-reduce across 640MB of parameters. At our measured 1\% synchronization rate, CurvaDion skips the 70ms communication cost in 99\% of steps while incurring only the 0.5ms flag overhead, yielding a 99.3\% reduction in communication time (70ms $\rightarrow$ 0.5ms).

\begin{table}[t]
\centering
\label{tab:network_projection}
\begin{tabular}{lccc}
\toprule
\textbf{Network Type} & \textbf{Latency} & \textbf{Sync Time} & \textbf{Speedup} \\
\midrule
InfiniBand & $\sim$0.1ms & 70ms & \textbf{1.020$\times$} \\
10GbE & $\sim$10ms & 700ms & \textbf{1.202$\times$} \\
Cross-DC WAN & $\sim$50ms & 12000ms & \textbf{4.361$\times$} \\
\bottomrule
\end{tabular}
\caption{\textbf{Projected performance across network configurations.} CurvaDion ($\tau$=0.5) achieves 99.24\% communication reduction (4.7 MB/step vs 619.0 MB/step baseline). Speedup calculated assuming 3375ms computation, 39ms optimizer overhead, and 0.5ms flag synchronization. See Appendix~\ref{app:wall_clock_calc} for detailed calculations.}
\end{table}

Table~\ref{tab:network_projection} demonstrates how CurvaDion's benefits scale with network constraints, projecting performance across three representative network configurations used in production distributed training. The projected synchronization times account for transferring 619 MB of data per step (our 160M parameter model), incorporating both network bandwidth limitations and round-trip latency costs. For InfiniBand/RoCE interconnects common in tightly-coupled GPU clusters~\cite{jeon2023network, li2020pytorch}, we measure 70ms synchronization time, yielding 1.020$\times$ speedup. This modest improvement directly reflects the network characteristics: gradient synchronization represents only 2\% of total step time (70ms out of 3484ms), with computation dominating at 96.9\%. However, for 10 Gigabit Ethernet networks widely deployed in mid-scale training clusters~\cite{jeon2023network}, bandwidth constraints increase synchronization to approximately 700ms per step (17\% of step time), amplifying CurvaDion's speedup to 1.203$\times$. The gains become substantial for cross-datacenter training scenarios~\cite{douillard2023diloco, mahajan2022gandiva}, where high latency (50-100ms RTT) and limited inter-datacenter bandwidth push synchronization costs to $\sim$12 seconds per step. At this network speed, communication dominates training time (78\% of baseline step time), and CurvaDion's 99.24\% communication reduction delivers 4.396$\times$ speedup by transmitting only 4.7 MB instead of 619 MB per step.

These projections demonstrate that CurvaDion's performance improvements scale directly with both model size and training duration. As gradient volumes increase with parameter count, communication costs grow proportionally across all network configurations, amplifying the relative benefit of adaptive synchronization. Importantly, these efficiency gains incur no algorithmic cost: convergence quality remains identical to full synchronous training (Table~\ref{tab:scale_results}), while the adaptive triggering mechanism introduces negligible overhead (0.5~ms flag all-reduce, 0.015\% of step time).

An important property of our momentum-based triggering mechanism is its robustness to hyperparameter variations. We conducted additional experiments varying global batch sizes ($\{512, 1024, 2048, 4096, 8192\}$ tokens) and Dion rank fractions ($r \in \{1/16, 1/8, 1/4\}$), presented in Appendix~\ref{app:hyperparameter_sweeps}. Across these configurations, CurvaDion maintained consistent synchronization behavior and convergence quality without requiring threshold recalibration. This robustness stems from the adaptive nature of the RMMC criterion: as batch size increases (reducing gradient noise) or rank fraction changes (affecting update quality), the momentum dynamics naturally reflect these regime shifts, automatically adjusting synchronization frequency to match the optimization landscape rather than requiring manual tuning for each configuration.

\section{Related Work}
Standard data-parallel training requires full gradient synchronization at every step~\citep{li2020pytorch, goyal2017accurate}, accumulating substantial communication overhead as models scale. This bottleneck has driven extensive research into communication reduction. Gradient compression methods reduce per-step volume through quantization~\citep{alistarh2017qsgd, seide20141bit, wen2017terngrad, bernstein2018signsgd} and sparsification~\citep{lin2018deep, stich2018sparsified, wangni2018gradient}, often with error feedback~\citep{karimireddy2019error} to maintain convergence. Low-rank approximations including PowerSGD~\citep{vogels2019powersgd}, GradZip~\citep{agarwal2020gradzip}, and Dion~\citep{ahn2025dion} achieve further compression through subspace projection. Orthogonal approaches reduce synchronization frequency: Local SGD~\citep{stich2018local, lin2020dont, yu2019parallel} and federated methods~\citep{mcmahan2017fedavg, karimireddy2020scaffold} perform local updates between periodic coordination, while asynchronous methods~\citep{dean2012large, recht2011hogwild, lian2015asynchronous} eliminate blocking entirely at the cost of stale gradients. Recent work extends periodic synchronization to cross-datacenter training~\citep{douillard2023diloco} and proposes adaptive schedules based on update significance~\citep{Tyagi_2023} or quadratic synchronization rules~\citep{ortiz2023quadratic}, though these require careful per-configuration tuning. Other distributed optimization methods achieve extreme communication reduction through architecture-agnostic approaches: DisTrO~\citep{peng2024distro} reduces inter-GPU communication by 4-5 orders of magnitude while matching standard convergence, and DeMo~\citep{safaryan2024demo} leverages decentralized momentum to enable training over bandwidth-constrained networks. 

While these methods optimize communication patterns, CurvaDion takes a fundamentally different approach by leveraging optimizer internals to detect when synchronization is necessary. Classical momentum methods~\citep{polyak1964some, nesterov1983method, sutskever2013importance} reduce gradient variance through exponential averaging, and adaptive optimizers including AdaGrad~\citep{duchi2011adaptive}, RMSprop~\citep{tieleman2012rmsprop}, Adam~\citep{kingma2015adam}, AdamW~\citep{loshchilov2019adamw}, and recent variants~\citep{chen2023lion, shazeer2018adafactor} adjust per-parameter learning rates, yet these optimizer dynamics have not been exploited to inform distributed synchronization decisions. Muon~\citep{jordan2024muon} applies orthogonalization to momentum updates, achieving strong results in LLM training but while momentum-based optimizers have proven effective, understanding when they traverse difficult regions of the loss landscape remains challenging. Loss landscape geometry, particularly curvature, fundamentally impacts both generalization~\citep{hochreiter1997flat, keskar2017large} and training dynamics~\citep{gilmer2021loss, cohen2021gradient, li2018visualizing}, as explicitly demonstrated by sharpness-aware methods~\citep{foret2021sharpnessaware, kwon2021asam, kunstner2024momentum} that seek flat minima. Second-order methods~\citep{martens2015optimizing, grosse2016kroneckerfactored, yao2021adahessian, goldfarb2020practical} can approximate curvature for better conditioning but incur prohibitive computational overhead for frequent evaluation during training. CurvaDion bridges this gap by recognizing that momentum dynamics themselves encode information about directional curvature along the optimization trajectory enabling geometry-aware synchronization through a zero-cost proxy that requires no second-order computation, unlike compression methods that synchronize every step or fixed schedules that risk divergence in high-curvature regions.

\section{Discussion}
CurvaDion demonstrates that momentum dynamics, already computed during optimization, provide a computationally tractable proxy for detecting high-curvature regions requiring worker synchronization. Our results establish that RMMC-based synchronization achieves 99\% communication reduction while maintaining convergence quality equivalent to full synchronous training across models from 160M to 1.3B parameters trained on up to 100B tokens. The key insight is that synchronization timing matters more than frequency: at matched 1\% synchronization rates, CurvaDion achieves 28.97 validation perplexity while fixed-interval synchronization degrades to 34.08, a difference attributable to adaptive placement during optimization-critical phases. This geometry-aware approach adds negligible overhead (0.5ms flag synchronization per step) while concentrating expensive communication operations precisely when the loss landscape demands coordination. The performance improvements scale directly with network constraints, ranging from 1.02$\times$ speedup on fast InfiniBand to projected 4.36$\times$ speedup for cross-datacenter training. This enables LLM training in bandwidth-constrained environments previously impractical due to communication costs. By eliminating 99\% of synchronization overhead without algorithmic compromises, CurvaDion reduces both operational costs and wall clock time for large-scale training, enabling scalability of trillion-parameter LLMs across geographically distributed datacenters where network bandwidth and latency would otherwise prohibit training.

\newpage
\clearpage
\onecolumn

\bibliographystyle{plainnat}
\bibliography{example_paper}

\newpage
\clearpage
\onecolumn
\appendix

\section{Theoretical Proofs}
\label{app:proof-theorem1}

We provide a complete proof of Theorem~\ref{thm:rmmc-curvature}, establishing the relationship between RMMC and directional curvature. We begin by restating the theorem for completeness.

\setcounter{theorem}{0}
\begin{theorem}[RMMC-Curvature Relationship, restated]
Under $L$-Lipschitz continuous gradients, $\|\nabla F(x) - \nabla F(y)\| \leq L\|x-y\|$ for all $x, y$, and small step size satisfying $\eta\|M_{t-1}\| = o(1/L)$, the Relative Maximum Momentum Change satisfies:
\begin{equation}
\text{RMMC}_t = \left|\eta \kappa_M(x_{t-1}) + B_t\right| + O(\eta^2)
\end{equation}
where $\kappa_M(x) = v^T \nabla^2 F(x) v$ is the directional curvature in the momentum direction $v = M/\|M\|$, and $B_t = (1-\mu) - \frac{v_{t-1}^T \nabla F(x_{t-1})}{\|M_{t-1}\|}$.
\end{theorem}

\begin{proof}
We decompose the momentum vector at time $t$ into its magnitude and direction. Let $m_t = \|M_t\|$ denote the momentum magnitude and $v_t = \frac{M_t}{\|M_t\|}$ denote the unit momentum direction, so that:
\begin{equation*}
M_t = m_t v_t
\end{equation*}
Similarly, $M_{t-1} = m_{t-1} v_{t-1}$ where $m_{t-1} = \|M_{t-1}\|$ and $v_{t-1} = \frac{M_{t-1}}{\|M_{t-1}\|}$ with $\|v_{t-1}\| = 1$.

The momentum update rule from Dion \cite{ahn2025dion} is $M_t = \mu M_{t-1} + \nabla F(x_t)$. To analyze how the magnitude changes, we compute the squared norm:
\begin{align*}
\|M_t\|^2 &= \|\mu M_{t-1} + \nabla F(x_t)\|^2 \\
&= (\mu M_{t-1} + \nabla F(x_t))^T(\mu M_{t-1} + \nabla F(x_t)) \\
&= \mu^2 \|M_{t-1}\|^2 + 2\mu M_{t-1}^T \nabla F(x_t) + \|\nabla F(x_t)\|^2 \\
&= \mu^2 m_{t-1}^2 + 2\mu M_{t-1}^T \nabla F(x_t) + \|\nabla F(x_t)\|^2
\end{align*}

We factor out $m_{t-1}^2$ to isolate the relative change:
\begin{align*}
m_t^2 &= m_{t-1}^2 \left[\mu^2 + \frac{2\mu M_{t-1}^T \nabla F(x_t)}{m_{t-1}^2} + \frac{\|\nabla F(x_t)\|^2}{m_{t-1}^2}\right]
\end{align*}

Taking the square root:
\begin{align*}
m_t &= m_{t-1} \sqrt{\mu^2 + \frac{2\mu M_{t-1}^T \nabla F(x_t)}{m_{t-1}^2} + \frac{\|\nabla F(x_t)\|^2}{m_{t-1}^2}}
\end{align*}

Under our small step size assumption, $\|\nabla F(x_t)\| = O(m_{t-1})$ and the terms inside the square root satisfy $\left|\frac{2\mu M_{t-1}^T \nabla F(x_t)}{m_{t-1}^2}\right|, \left|\frac{\|\nabla F(x_t)\|^2}{m_{t-1}^2}\right| = O(1)$. We apply the Taylor approximation $\sqrt{a^2 + \epsilon} = |a|\sqrt{1 + \epsilon/a^2} \approx |a|(1 + \frac{\epsilon}{2a^2})$ for small $\epsilon/a^2$:
\begin{align*}
m_t &\approx m_{t-1}\mu \left[1 + \frac{1}{2\mu^2}\left(\frac{2\mu M_{t-1}^T \nabla F(x_t)}{m_{t-1}^2} + \frac{\|\nabla F(x_t)\|^2}{m_{t-1}^2}\right)\right] \\
&= m_{t-1}\mu + \frac{M_{t-1}^T \nabla F(x_t)}{m_{t-1}} + \frac{\|\nabla F(x_t)\|^2}{2\mu m_{t-1}}
\end{align*}

The final term $\frac{\|\nabla F(x_t)\|^2}{2\mu m_{t-1}}$ is second-order in the step size. To see this, note that $\|\nabla F(x_t)\| = O(\|M_{t-1}\|)$ in typical training and $\|M_{t-1}\| \sim \frac{\|\nabla F\|}{1-\mu}$, giving $\frac{\|\nabla F(x_t)\|^2}{m_{t-1}} = O(\|\nabla F(x_t)\|) = O(\eta m_{t-1})$ after one step. Thus:
\begin{align*}
m_t - m_{t-1} &= \mu m_{t-1} - m_{t-1} + \frac{M_{t-1}^T \nabla F(x_t)}{m_{t-1}} + O(\eta^2 m_{t-1}) \\
&= (\mu - 1)m_{t-1} + \frac{M_{t-1}^T \nabla F(x_t)}{m_{t-1}} + O(\eta^2 m_{t-1})
\end{align*}

Since $M_{t-1} = m_{t-1} v_{t-1}$, we have $M_{t-1}^T = m_{t-1} v_{t-1}^T$, giving:
\begin{align*}
m_t - m_{t-1} &= (\mu - 1)m_{t-1} + \frac{m_{t-1} v_{t-1}^T \nabla F(x_t)}{m_{t-1}} + O(\eta^2 m_{t-1}) \\
&= (\mu - 1)m_{t-1} + v_{t-1}^T \nabla F(x_t) + O(\eta^2 m_{t-1})
\end{align*}

\textbf{Taylor expansion of the gradient $\nabla F(x_t)$}

The parameter update is $x_t = x_{t-1} - \eta M_{t-1}$, so $x_t - x_{t-1} = -\eta M_{t-1}$. Under $L$-Lipschitz continuous gradients, the fundamental theorem of calculus gives:
\begin{align*}
\nabla F(x_t) &= \nabla F(x_{t-1}) + \int_0^1 \nabla^2 F(x_{t-1} + s(x_t - x_{t-1}))(x_t - x_{t-1})\,ds
\end{align*}

We separate the Hessian at $x_{t-1}$ from the remainder:
\begin{align*}
\nabla F(x_t) &= \nabla F(x_{t-1}) + \nabla^2 F(x_{t-1})(x_t - x_{t-1}) \\
&+ \int_0^1 [\nabla^2 F(x_{t-1} + s(x_t - x_{t-1})) - \nabla^2 F(x_{t-1})](x_t - x_{t-1})\,ds
\end{align*}

The remainder term satisfies $\|R\| \leq \int_0^1 Ls\|x_t - x_{t-1}\|^2\,ds = \frac{L}{2}\|x_t - x_{t-1}\|^2$ by the Lipschitz condition. Substituting $x_t - x_{t-1} = -\eta M_{t-1}$:
\begin{align*}
\nabla F(x_t) &= \nabla F(x_{t-1}) + \nabla^2 F(x_{t-1})(-\eta M_{t-1}) + R \\
&= \nabla F(x_{t-1}) - \eta \nabla^2 F(x_{t-1}) M_{t-1} + O(\eta^2 \|M_{t-1}\|^2) \\
&= \nabla F(x_{t-1}) - \eta \nabla^2 F(x_{t-1}) M_{t-1} + O(\eta^2 m_{t-1}^2)
\end{align*}

\textbf{Projecting onto the momentum direction and extracting curvature.}

We substitute the gradient expansion into the momentum change
\begin{align*}
v_{t-1}^T \nabla F(x_t) &= v_{t-1}^T [\nabla F(x_{t-1}) - \eta \nabla^2 F(x_{t-1}) M_{t-1}] + O(\eta^2 m_{t-1}) \\
&= v_{t-1}^T \nabla F(x_{t-1}) - \eta v_{t-1}^T \nabla^2 F(x_{t-1}) M_{t-1} + O(\eta^2 m_{t-1})
\end{align*}

Using $M_{t-1} = m_{t-1} v_{t-1}$:
\begin{align*}
v_{t-1}^T \nabla^2 F(x_{t-1}) M_{t-1} &= v_{t-1}^T \nabla^2 F(x_{t-1}) (m_{t-1} v_{t-1}) \\
&= m_{t-1} (v_{t-1}^T \nabla^2 F(x_{t-1}) v_{t-1})
\end{align*}

We define the directional curvature in the momentum direction as:
\begin{equation*}
\kappa_M(x_{t-1}) := v_{t-1}^T \nabla^2 F(x_{t-1}) v_{t-1}
\end{equation*}

This scalar quantity measures the second derivative of $F$ along the momentum direction. Substituting back:
\begin{align*}
v_{t-1}^T \nabla F(x_t) &= v_{t-1}^T \nabla F(x_{t-1}) \\
&- \eta m_{t-1} \kappa_M(x_{t-1}) + O(\eta^2 m_{t-1})
\end{align*}

Therefore:
\begin{align*}
m_t - m_{t-1} &= (\mu - 1)m_{t-1} + v_{t-1}^T \nabla F(x_{t-1})\\ 
&- \eta m_{t-1} \kappa_M(x_{t-1}) + O(\eta^2 m_{t-1})
\end{align*}

Dividing both sides by $m_{t-1}$:
\begin{align*}
\frac{m_t - m_{t-1}}{m_{t-1}} &= (\mu - 1) + \frac{v_{t-1}^T \nabla F(x_{t-1})}{m_{t-1}} \\
&- \eta \kappa_M(x_{t-1}) + O(\eta^2)
\end{align*}

Rearranging terms:
\begin{align*}
\frac{m_t - m_{t-1}}{m_{t-1}} &= -\eta \kappa_M(x_{t-1}) \\
&+ \left[(\mu - 1) + \frac{v_{t-1}^T \nabla F(x_{t-1})}{m_{t-1}}\right] + O(\eta^2)
\end{align*}

Since $\mu - 1 = -(1-\mu)$, we define the bias term:
\begin{equation*}
B_t := (1-\mu) - \frac{v_{t-1}^T \nabla F(x_{t-1})}{m_{t-1}}
\end{equation*}

This gives:
\begin{align*}
\frac{m_t - m_{t-1}}{m_{t-1}} &= -\eta \kappa_M(x_{t-1}) - B_t + O(\eta^2)
\end{align*}

Taking the absolute value and using the definition $\text{RMMC}_t = \left|\frac{m_t - m_{t-1}}{m_{t-1}}\right|$:
\begin{align*}
\text{RMMC}_t &= \left|-\eta \kappa_M(x_{t-1}) - B_t\right| + O(\eta^2) \\
&= \left|\eta \kappa_M(x_{t-1}) + B_t\right| + O(\eta^2)
\end{align*}

\end{proof}

\newpage
\section{Full Training Curves at Scale}
\label{app:scaling_runs}

We present complete training and validation curves for the scaling test. All experiments use $\tau = 0.5$, learning rate $\eta = 0.02$ with 1\% warmup and 10\% warmdown, momentum $\mu = 0.95$, weight decay 0.01, and Dion rank fraction $r = 0.125$. Training configurations are summarized in Table~\ref{tab:scaling_configs}.

\begin{table}[h]
\centering
\begin{tabular}{lcccc}
\toprule
\textbf{Model} & \textbf{$d_{\text{model}}$ / Layers / Heads} & \textbf{Steps} & \textbf{Total Tokens} & \textbf{Batch Size} \\
\midrule
160M & 768 / 12 / 6 & 3,000 & 3.1B & 1.0M \\
350M & 1024 / 24 / 32 & 9,000 & 18.9B & 2.1M \\
1.3B & 2048 / 24 / 32 & 47,000 & 98.7B & 2.1M \\
\bottomrule
\end{tabular}
\caption{Scaling experiment configurations.}
\label{tab:scaling_configs}
\end{table}

\subsection{160M Parameter Model}

\begin{figure}[h]
\centering
\includegraphics[width=0.48\textwidth]{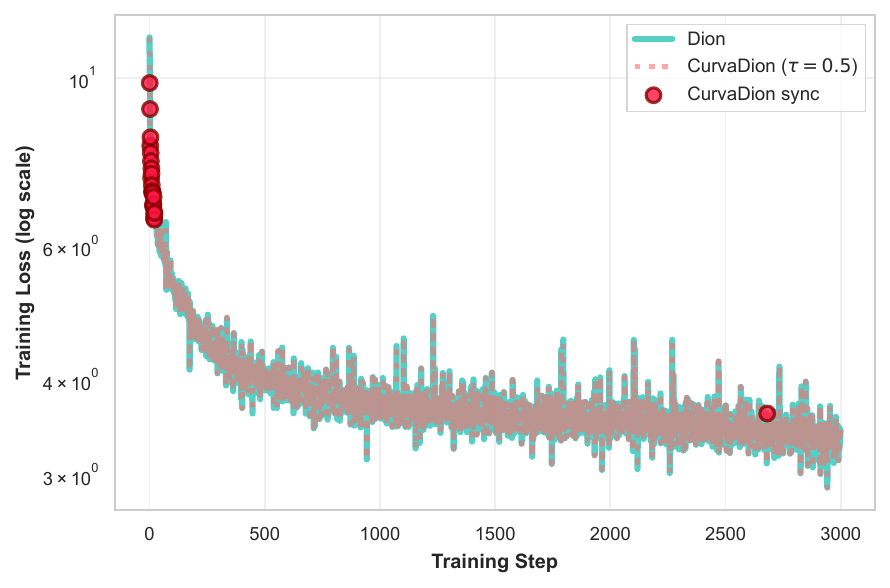}
\includegraphics[width=0.48\textwidth]{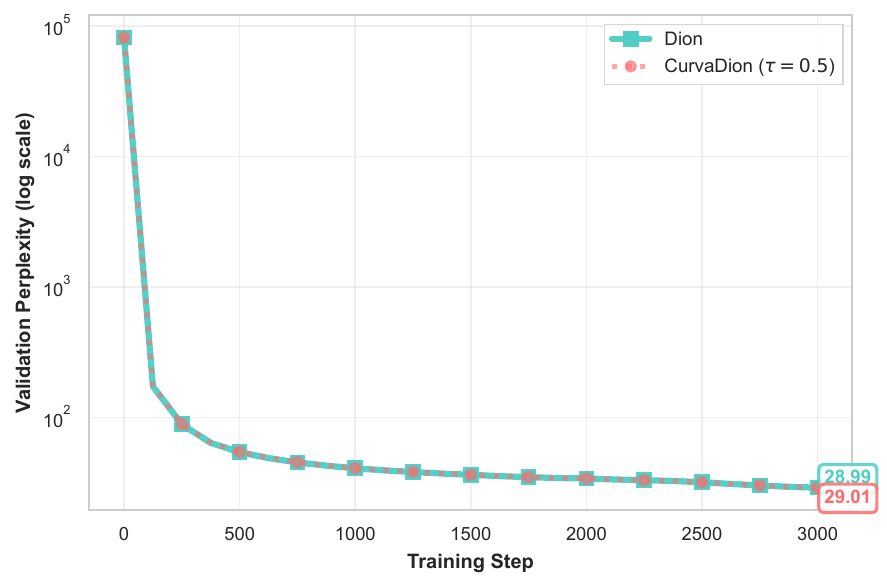}
\caption{Training loss (left) and validation perplexity (right) for 160M model. CurvaDion synchronizations (red markers) concentrate during early training and high-curvature phases, achieving final validation perplexity of 29.01 compared to Dion's 28.99.}
\label{fig:160m_curves}
\end{figure}

\subsection{360M Parameter Model}

\begin{figure}[h]
\centering
\includegraphics[width=0.48\textwidth]{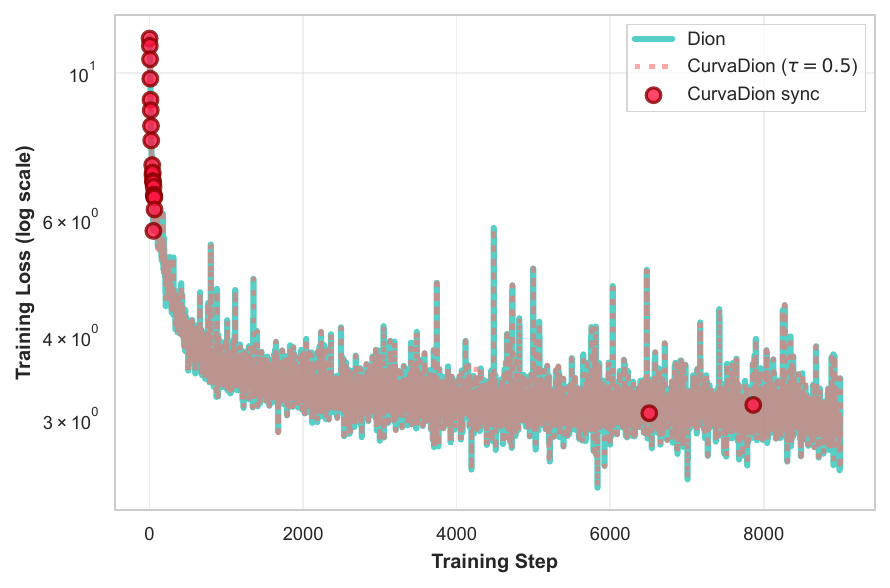}
\includegraphics[width=0.48\textwidth]{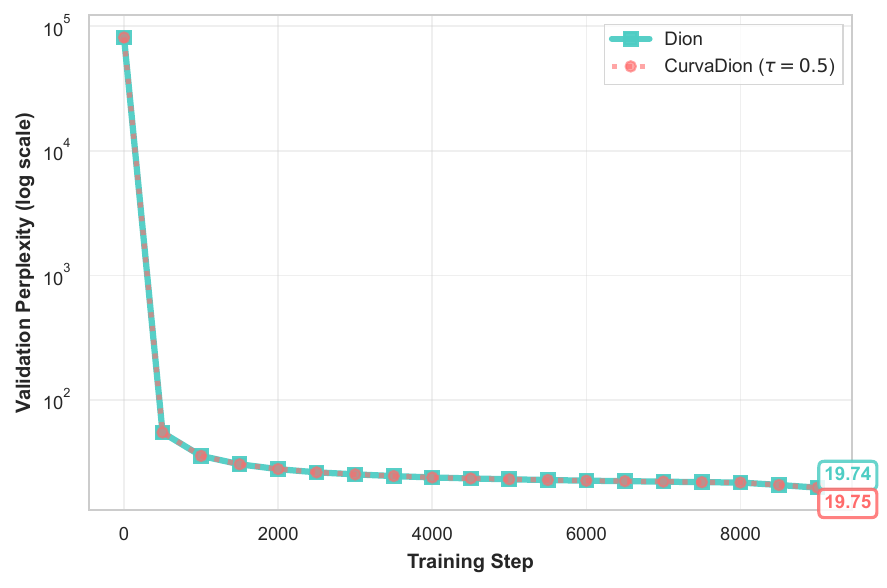}
\caption{Training loss (left) and validation perplexity (right) for 360M model. CurvaDion maintains convergence quality matching baseline Dion (19.75 vs. 19.74 final perplexity) while adaptively concentrating synchronization during optimization-critical phases.}
\label{fig:360m_curves}
\end{figure}

\subsection{1.3B Parameter Model}

\begin{figure}[h]
\centering
\includegraphics[width=0.48\textwidth]{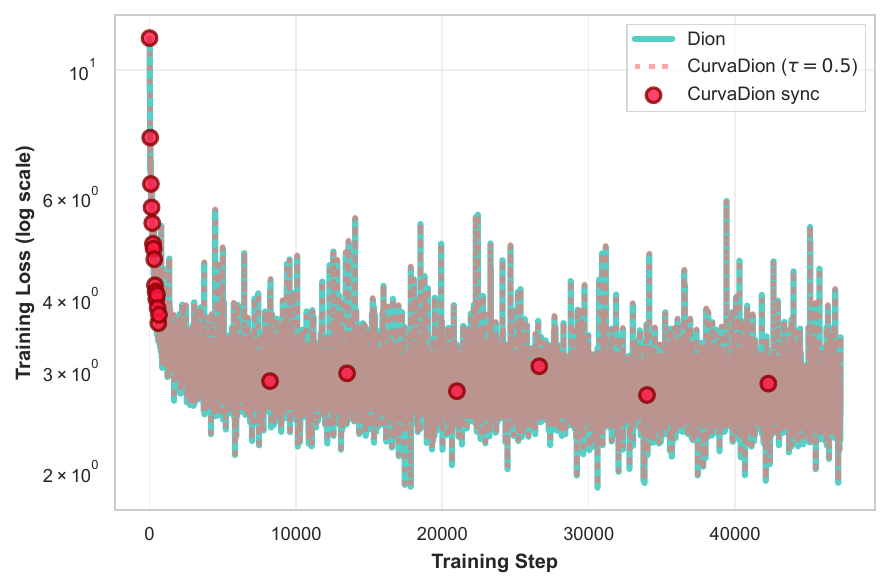}
\includegraphics[width=0.48\textwidth]{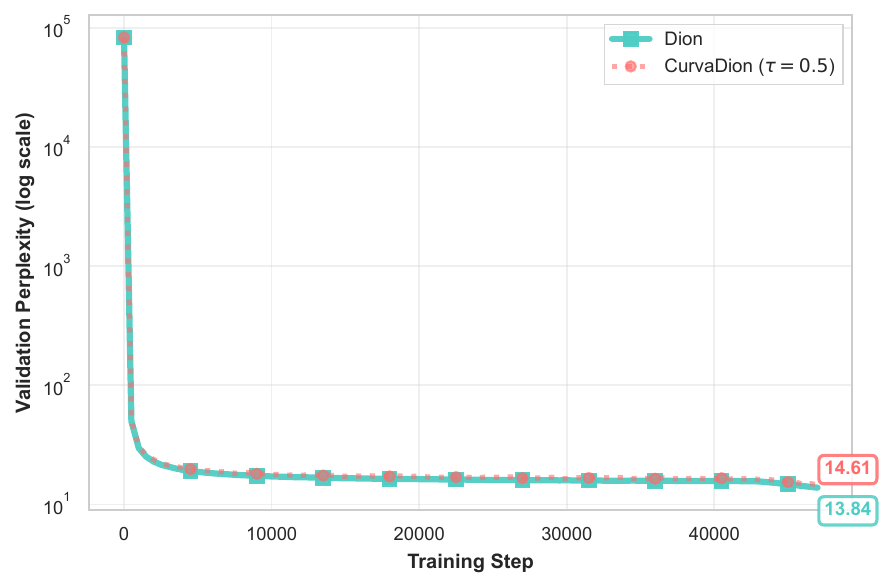}
\caption{Training loss (left) and validation perplexity (right) for 1.3B model. CurvaDion maintains convergence quality matching baseline Dion (13.84 vs. 14.61 final perplexity) while adaptively concentrating synchronization during optimization-critical phases.}
\label{fig:360m_curves}
\end{figure}

\newpage

\newpage
\newpage
\section{Scheduled Synchronization Comparison}
\label{app:scheduled_sync}

We compare CurvaDion against Scheduled Sync Dion across different synchronization frequencies to isolate the effect of adaptive placement versus total communication budget. Scheduled Sync Dion synchronizes workers at fixed intervals every $H$ steps, while CurvaDion adaptively triggers synchronization based on RMMC threshold $\tau$. All experiments use the 160M parameter model trained for 3000 iterations with learning rate $\eta = 0.02$ with 1\% warmup and 20\% warmdown, momentum $\mu = 0.95$, weight decay 0.01, and Dion rank fraction $r = 0.125$.

\subsection{High Synchronization Rate: $H=10$ vs $\tau=0.1$}

\begin{figure}[h]
\centering
\includegraphics[width=0.48\textwidth]{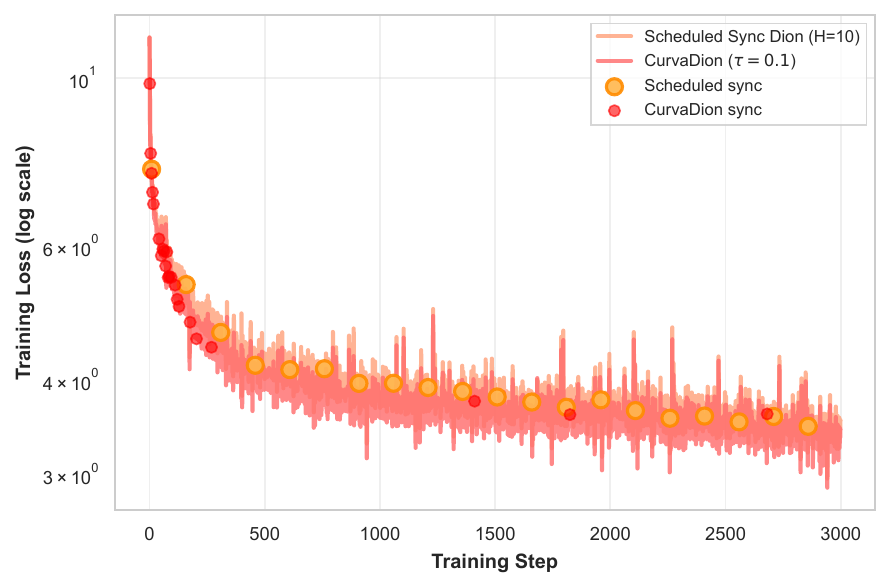}
\includegraphics[width=0.48\textwidth]{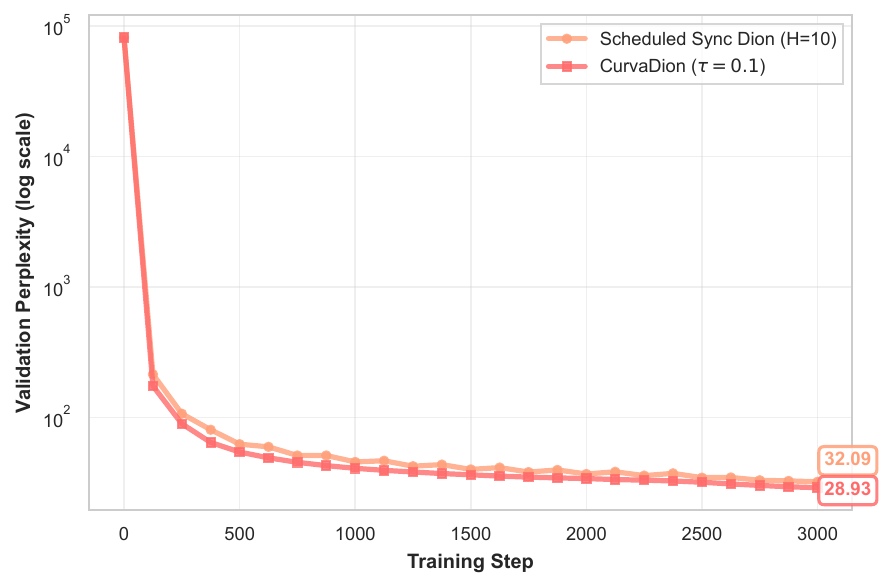}
\caption{Training loss (left) and validation perplexity (right) comparing Scheduled Sync Dion ($H=10$, 10\% sync rate) with CurvaDion ($\tau=0.1$). CurvaDion achieves 28.93 final perplexity compared to 32.09 for scheduled synchronization, demonstrating superior convergence despite similar communication frequency through adaptive placement during critical optimization phases.}
\label{fig:scheduled_h10}
\end{figure}

\subsection{Medium Synchronization Rate: $H=50$ vs $\tau=0.2$}

\begin{figure}[h]
\centering
\includegraphics[width=0.48\textwidth]{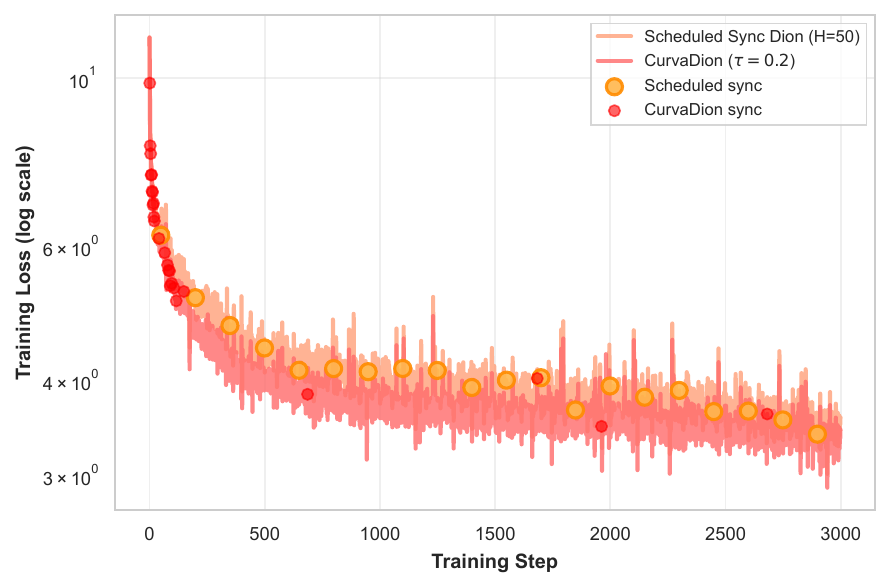}
\includegraphics[width=0.48\textwidth]{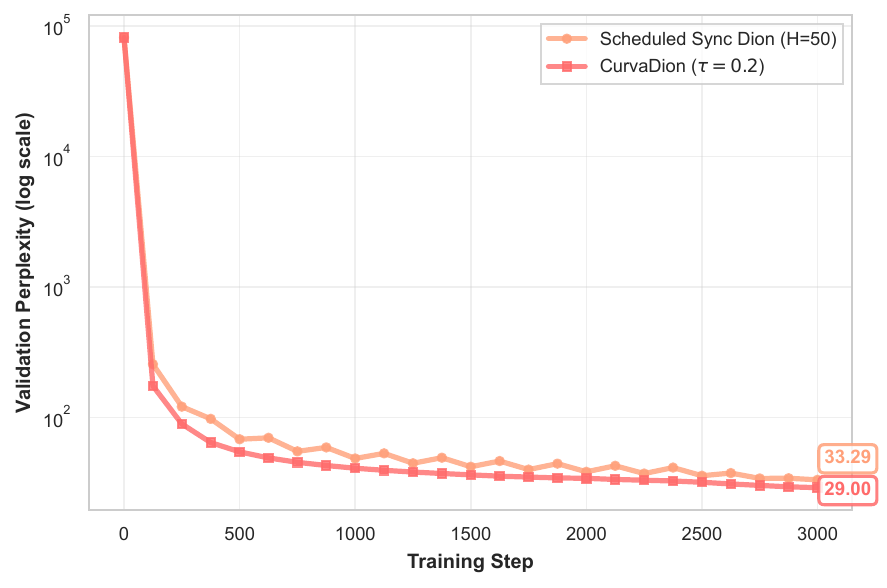}
\caption{Training loss (left) and validation perplexity (right) comparing Scheduled Sync Dion ($H=50$, 2\% sync rate) with CurvaDion ($\tau=0.2$). The performance gap widens to 4.29 perplexity points (33.29 vs. 29.00), as fixed-interval synchronization increasingly misses critical high-curvature regions that require worker coordination.}
\label{fig:scheduled_h50}
\end{figure}

\subsection{Low Synchronization Rate: $H=250$ vs $\tau=0.9$}

\begin{figure}[h]
\centering
\includegraphics[width=0.48\textwidth]{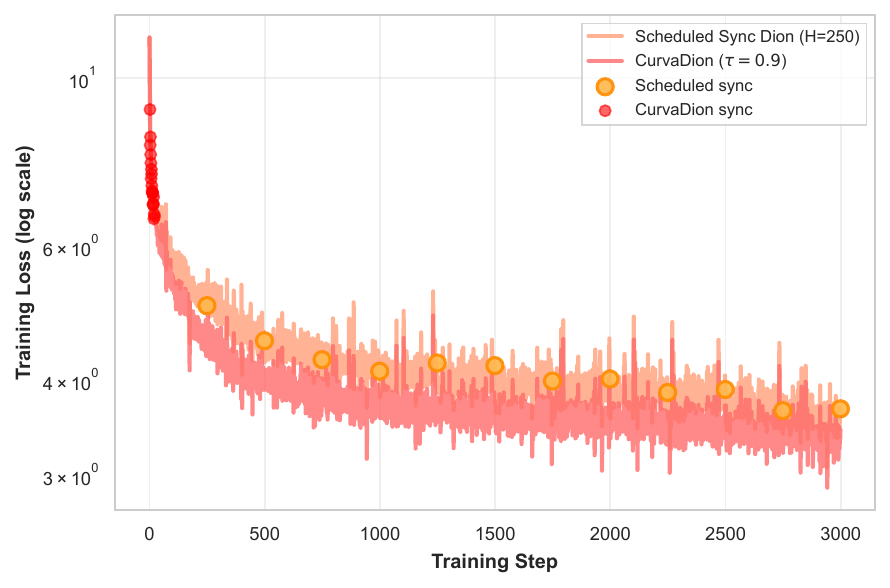}
\includegraphics[width=0.48\textwidth]{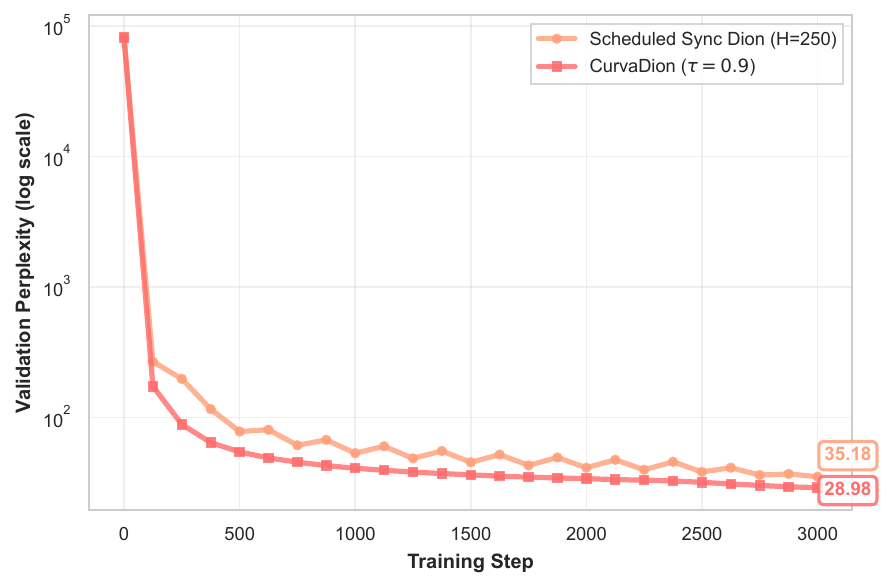}
\caption{Training loss (left) and validation perplexity (right) comparing Scheduled Sync Dion ($H=250$, 0.4\% sync rate) with CurvaDion ($\tau=0.9$). At very low synchronization rates, Scheduled Sync Dion diverges significantly (35.18 perplexity), while CurvaDion maintains convergence quality (28.98 perplexity) by concentrating its limited communication budget during optimization-critical phases.}
\label{fig:scheduled_h250}
\end{figure}

These results reveal that increasing scheduled synchronization frequency (smaller $H$) reduces worker divergence, while decreasing frequency (larger $H$) leads to progressively worse convergence (32.09 $\rightarrow$ 33.29 $\rightarrow$ 35.18 perplexity). In contrast, CurvaDion maintains consistent validation perplexity ($\sim$29) across all threshold values by adaptively placing synchronizations during high-curvature regions, demonstrating that when synchronization occurs matters more than how often it occurs.

\newpage
\section{Hyperparameter Robustness}
\label{app:hyperparameter_sweeps}

We evaluate CurvaDion's robustness across varying batch sizes and Dion rank fractions using the 160M parameter model trained for 3000 iterations with $\tau = 0.5$. All experiments maintain identical hyperparameters from Section~4.1 except for the swept parameter.

\subsection{Batch Size Variations}

\begin{figure}[h]
\centering
\includegraphics[width=0.48\textwidth]{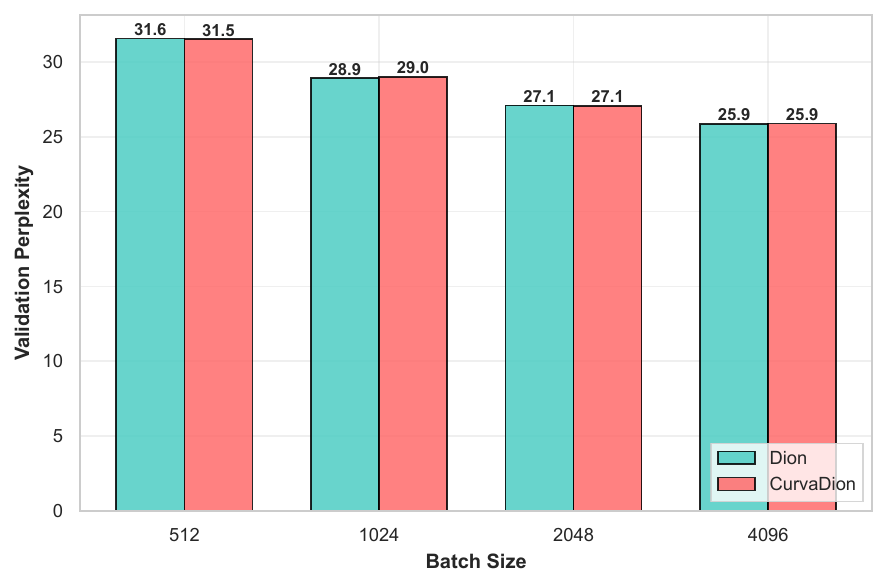}
\includegraphics[width=0.48\textwidth]{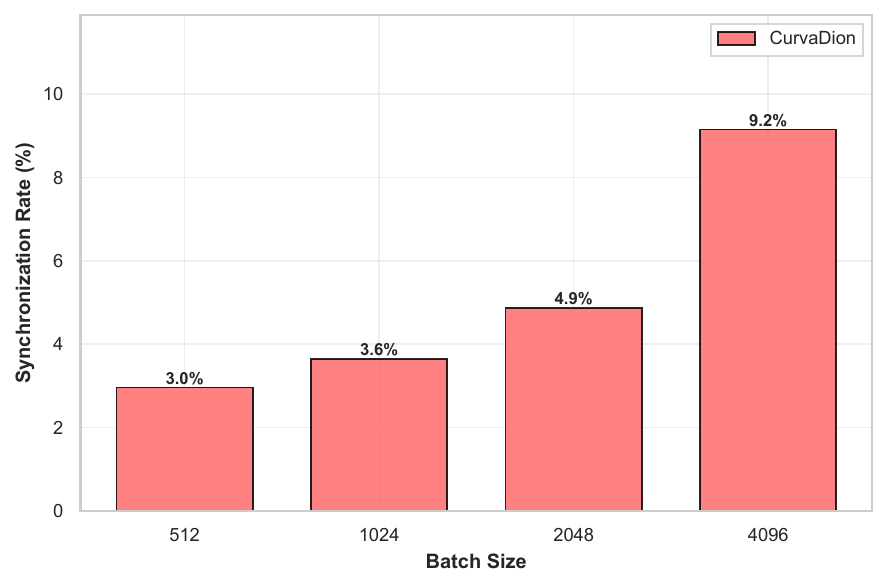}
\caption{CurvaDion maintains convergence quality matching baseline Dion across batch sizes $\{512, 1024, 2048, 4096\}$ tokens (left). Synchronization rate adapts naturally, increasing from 3.0\% to 9.2\% as larger batches reduce gradient noise and enable more aggressive optimization trajectories (right).}
\label{fig:batch_sweeps}
\end{figure}

\subsection{Rank Fraction Variations}

\begin{figure}[h]
\centering
\includegraphics[width=0.48\textwidth]{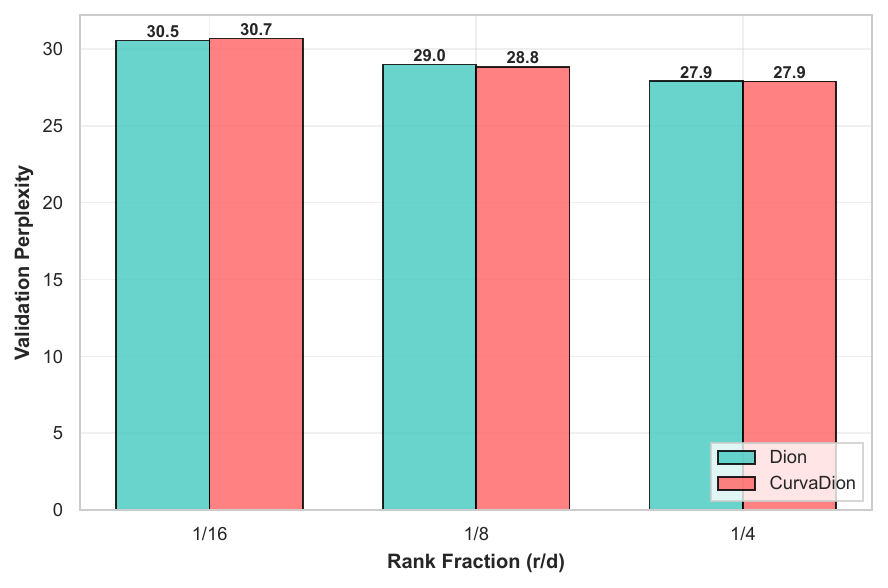}
\includegraphics[width=0.48\textwidth]{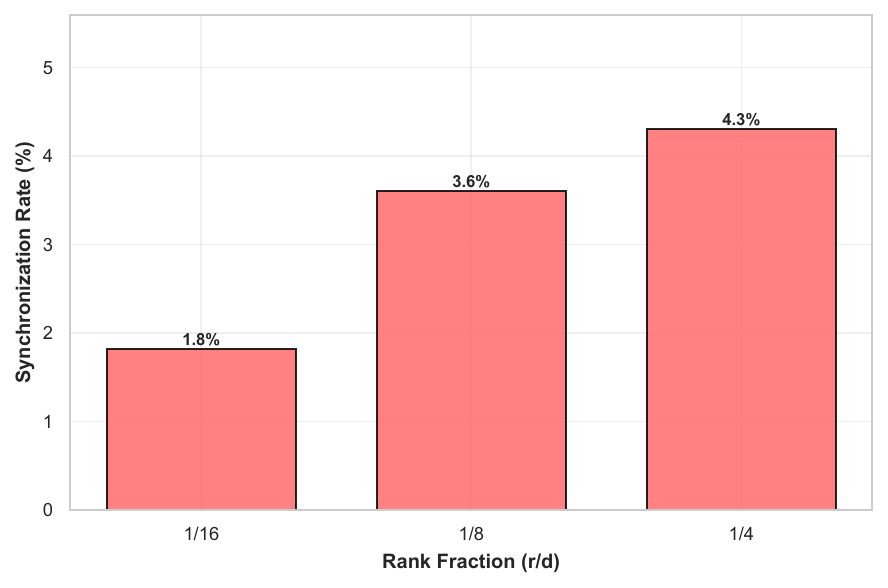}
\caption{CurvaDion matches baseline Dion convergence across rank fractions $r \in \{1/16, 1/8, 1/4\}$ (left). Synchronization rate increases from 1.8\% to 4.3\% with higher rank, reflecting richer updates that traverse the loss landscape more aggressively (right).}
\label{fig:rank_sweeps}
\end{figure}

These results demonstrate that CurvaDion's momentum-based triggering mechanism adapts automatically to different optimization regimes without requiring threshold recalibration.

\newpage
\section{Wall Clock Calculation Details}
\label{app:wall_clock_calc}

We calculate projected speedups using measured timings from Section~\ref{sec:wall_clock_analysis}. CurvaDion achieves approximately 1\% synchronization rate, performing full gradient synchronization in only 1\% of training steps. Each step consists of: 3375~ms computation, 39~ms optimizer update, and network-dependent communication. Baseline Dion synchronizes 619~MB every step. CurvaDion incurs 0.5~ms flag synchronization every step, plus full 619~MB synchronization in 1\% of steps.

\paragraph{InfiniBand/RoCE ($\sim$100~Gb/s):}
Communication time: 70~ms per 619~MB synchronization.
\begin{align*}
T_{\text{Dion}} &= 3375 + 39 + 70 = 3484 \text{ ms} \\
T_{\text{CurvaDion}} &= 3375 + 39 + 0.5 + (0.01 \times 70) = 3415.2 \text{ ms} \\
\text{Speedup} &= \frac{3484}{3415.2} = \mathbf{1.020\times}
\end{align*}

\paragraph{10 Gigabit Ethernet ($\sim$10~Gb/s):}
Communication time: 700~ms per 619~MB synchronization.
\begin{align*}
T_{\text{Dion}} &= 3375 + 39 + 700 = 4114 \text{ ms} \\
T_{\text{CurvaDion}} &= 3375 + 39 + 0.5 + (0.01 \times 700) = 3421.5 \text{ ms} \\
\text{Speedup} &= \frac{4114}{3421.5} = \mathbf{1.202\times}
\end{align*}

\paragraph{Cross-datacenter WAN ($\sim$1~Gb/s + 50ms RTT):}
Communication time: 12000~ms per 619~MB synchronization.
\begin{align*}
T_{\text{Dion}} &= 3375 + 39 + 12000 = 15414 \text{ ms} \\
T_{\text{CurvaDion}} &= 3375 + 39 + 0.5 + (0.01 \times 12000) = 3534.5 \text{ ms} \\
\text{Speedup} &= \frac{15414}{3534.5} = \mathbf{4.361\times}
\end{align*}

\end{document}